\documentclass[12pt]{article}
\usepackage{amsmath}
\usepackage{natbib}
\usepackage{graphicx}
\usepackage{amsfonts}
\usepackage{amssymb}
\usepackage{mathrsfs}
\usepackage{yfonts}
\usepackage{subfig}
\usepackage{enumitem}
\usepackage{multicol}
\usepackage{multirow}
\usepackage{graphicx}
\usepackage[toc,page]{appendix}
\usepackage{hyperref}
\usepackage{amsthm} 
\usepackage{color}
\usepackage{colortbl}
\usepackage{mathtools}
\usepackage{subfig}
\usepackage[export]{adjustbox}

\usepackage{float}
\usepackage{natbib}
\usepackage[english]{babel}

\usepackage{graphicx}
\usepackage{caption}

\usepackage{natbib}
\setcitestyle{authoryear,open={(},close={)}}

\newtheorem{Thm}{Theorem}
\newtheorem{prop}{Proposition}

\newtheorem{Assump}{Assumption}

\def\spacingset#1{\renewcommand{\baselinestretch}%
{#1}\small\normalsize} \spacingset{1}

\title{Semi-Supervised Empirical Risk Minimization: \\
Using unlabeled data to improve prediction}
\author{Oren Yuval \\Department of Statistics, Tel-Aviv University, Tel-Aviv, Israel, 69978\vspace{5mm} \\ Saharon Rosset\\Department of Statistics, Tel-Aviv University, Tel-Aviv, Israel, 69978 }

\begin{document}

\maketitle

\begin{abstract}
We present a general methodology for using unlabeled data to design semi supervised learning (SSL) variants of the Empirical Risk Minimization (ERM) learning process. Focusing on generalized linear regression, we analyze of the effectiveness of our SSL approach in improving prediction performance. The key ideas are carefully considering the null model as a competitor, and utilizing the unlabeled data to determine signal-noise combinations where SSL outperforms both supervised learning and the null model. We then use SSL in an \textit{adaptive} manner based on estimation of the signal and noise. 

In the special case of linear regression with Gaussian covariates, we prove that the non-adaptive SSL version is in fact not capable of improving on both the supervised estimator and the null model simultaneously, beyond a negligible $O(1/n)$ term. On the other hand, the adaptive model presented in this work, can achieve a substantial improvement over both competitors simultaneously, under a variety of settings. This is shown empirically through extensive simulations, and extended to other scenarios, such as non-Gaussian covariates, misspecified linear regression, or generalized linear regression with non-linear link functions.
\end{abstract}

\vspace{3mm}


\newpage
\spacingset{1.5}
\section{Introduction}
\subsection{Background and related work}
In many applications, labeled data used for learning processes can be much more expensive than unlabeled data. In the situation where a large amount of unlabeled is available but only a small amount of labeled data, we are motivated to use the unlabeled data to improve the prediction performance of a given supervised learning algorithm by applying semi-supervised learning (SSL) approaches. The question of whether unlabeled data are helpful and if so, how they can be effectively used, has been studied extensively. Some SSL methods and their effectiveness are described by \cite{zhou2014semi}, by \cite{singh2009unlabeled} and \cite{zhu2005semi}, mainly in the context of classification problems. More recent works \citep{kingma2014semi,salimans2016improved,da2019generative,sun2020deep,han2020unsupervised}, present new techniques to improve state-of-the-art deep generative models using unlabeled data. These methods involve complex architectures and require extensive computation. Moreover, theoretical analysis regarding their performance is hard to obtain.

Another class of related works \citep{javanmard2014confidence,javanmard2018debiasing,bellec2019biasing,bellec2018prediction}, discusses the usefulness of unlabeled data in the Lasso-regularized sparse linear regression. Of these, the most closely related to this work is by \cite{bellec2018prediction}, which proposes some new adaptations of the Lasso, and establishes oracle inequalities for the prediction performance under some restriction over the \textit{problem setting} (i.e. distribution of the covariates, sample size, dimensionality, and sparsity). On the other hand, only a few works have studied the usefulness of unlabeled data to improve the classical ordinary least squares (OLS) regression. \cite{tarpey2014paradoxical} presented a semi-supervised estimator $\tilde{\beta}$ that uses the known distribution of the covariates and compared its variance to that of the supervised standard OLS estimator $\hat{\beta}$. The conclusion was that $\tilde{\beta}$ outperforms $\hat{\beta}$ only if the data is noisy enough, or if the dimension of the problem is high enough. \cite{chakrabortty2018efficient} presented an algorithm for improving linear regression using an imputation step and a follow-up refitting step. This approach requires some assumptions on the effectiveness of the imputation step, such as asymptotic normality of the OLS estimator $\hat{\beta}$. In their recent work, \cite{zhang2019semi} proposed an estimator of the population mean that uses unlabeled data combined with the least squares method. The new estimator provides a shorter confidence interval compared to the traditional sample mean. In subsequent work, \cite{azriel2016semi} aim to improve the least squares estimator by transforming the regression problem into a mean estimation problem. As in \cite{chakrabortty2018efficient}, they show that the semi-supervised estimator can improve upon the standard OLS estimator only when the linear model is biased.

In this paper, we present a general approach for using unlabeled data in SSL for prediction tasks. This approach also yields a procedure that utilizes the unlabeled data to determine whether or not it is helpful to improve prediction performance under an arbitrary problem setting. We demonstrate the suggested methodology on generalized linear regression in the under-parameterized regime, by deriving dedicated formulas according to the model setting, and showing by simulations that these formulas are indeed accurate and practical. We also provide a theoretical result stating that in some cases (such as in linear regression with Gaussian covariates), the suggested SSL can improve either on the standard supervised estimator or the null model, that ignores the covariates, but not on both at once.


\subsection{Notations, assumptions, and main idea}
A statistical learning process seeks to fit a predictor $\hat{f}_T:\mathbb{R}^p \to \mathbb{R}$ that maps from a covariate vector $x\in \mathbb{R}^p$ to a predicted response $\hat{y}\in \mathbb{R}$, based on a training data set $T$. In the supervised setting: $T=(X,Y)$, where $X\in \mathbb{R}^{n\times p}$ and $Y \in \mathbb{R}^n$ such that $(x_1,y_1),...,(x_n,y_n)$ are i.i.d. according to some joint distribution $P_{xy}$, and we focus on the case: $n>p$. In the semi-supervised setting: $T=(X,Y,Z)$, where $Z\in \mathbb{R}^{m\times p}$ is the set of unlabeled data with $z_1,...,z_m$ i.i.d. observations from distribution $P_{x}$. For simplicity, we assume that the distribution $P_{x}$ is centered around zero, i.e., $\mathbb{E}[x] = \textbf{0}_p$. We also assume constant conditional variance, meaning that:
\begin{equation*}
    y= f(x)+\epsilon =  \mathbb{E}[y|x] +\epsilon \hspace{3mm};\hspace{3mm}
    \mathbb{V}ar(\epsilon) = \mathbb{V}ar(y|x) = \sigma^2.
\end{equation*}

For the sake of our theoretical analysis, we assume a \textit{total information} scenario which means that $m\to \infty$ and therefore we are able to estimate precisely $\mathbb{E}\left[\varphi (x) \right]$ for any function $\varphi$, by using the set $Z$:
\begin{equation*}
    \mathbb{E}_x\left[\varphi (x) \right] \approx \frac{1}{m} \sum_{i=1}^m \varphi (z_i),
\end{equation*}
and we assume the approximation is arbitrarily good. Thus, we study the question of whether or not the knowledge of $ \mathbb{E}_x\left[\varphi (x) \right]$ for some well-defined functions $\varphi$, can be used effectively to improve prediction. In the empirical simulations, however, we estimate the expected values by averaging $\varphi$ over a considerably large but finite set of unlabeled data.     

Denoting by $(x_0,y_0)$ an independent draw from $P_{xy}$, the learning process aims to minimize the out-of-sample prediction error, $R_{T} = \mathbb{E}_{x_0y_0}\left[ L\left(\hat{f}_T;x_0,y_0 \right) \right]$, according to some loss function $L$ that depends on $\hat{f}_T$. A common supervised learning process is the Empirical Risk Minimization (ERM) which suggests to minimize the loss over the training data:
\begin{equation*}
    \hat{f}_T = \underset{f\in \mathcal{F}}{argmin} \left\{  \frac{1}{n}\sum_{i=1}^n L\left( f;x_i,y_i \right)  \right\}, 
\end{equation*}
where $\mathcal{F}$ is a fixed class of functions. The prediction performance of the learning process is measured by the mean of $R_T$ over all possible training samples $T$, which is denoted by $R$.

The main idea we suggest here under the name \textit{semi-supervised ERM}, is to break the loss function into sum of elements that can be estimated separately. Then, we use the unlabeled data to better estimate any element in the loss function that depends only on the covariate vector $x$. For example, if the loss function can be written as: $L\left(\hat{f}_T;x,y \right) =L_1\left(\hat{f}_T;x \right)+ L_2\left(\hat{f}_T;x,y \right) ,$ the out-of-sample prediction error which is the objective for minimization can be viewed as: $R_T =  \mathbb{E}_{x_0}\left[ L_1\left(\hat{f}_T;x_0 \right) \right] +  \mathbb{E}_{x_0y_0}\left[ L_2\left(\hat{f}_T;x_0,y_0 \right) \right]$. The first term can be estimated separately and precisely by using the unlabeled data, and this leads us to suggest the following fitting as a semi-supervised ERM procedure:
\begin{equation} \label{gen_tilde}
    \tilde{f}_T = \underset{f\in \mathcal{F}}{argmin} \left\{ \mathbb{E}_{x} \left[ L_1\left(f;x \right)  \right]+ \frac{1}{n}\sum_{i=1}^n L_2\left( f;x_i,y_i \right)    \right\}.
\end{equation}
Assuming that the loss function can be decomposed even more, such as: $L\left(\hat{f}_T;x,y \right) =L_1\left(\hat{f}_T;x \right)+ L_{2,1}\left(\hat{f}_T;x \right)  L_{2,2}\left(y \right)$, then we can use the same logic, and decompose the expectation of product, $\mathbb{E}_{x_0y_0}\left[ L_{2,1}\left(\hat{f}_T;x_0 \right)  L_{2,2}\left(y_0 \right) \right]$, into product of expectations plus the covariance. By that, we utilize the unlabeled data to precisely estimate the term $ \mathbb{E}_{x} \left[ L_{2,1}\left(f;x \right)  \right]$, and this leads us to define a more sophisticated optimization problem for the semi-supervised ERM procedure: 
\begin{align} \label{gen_Breve}
    \Breve{f}_T = \underset{f\in \mathcal{F}}{argmin} &\Big\{ \mathbb{E}_{x} \left[ L_1\left(f;x \right)  \right]  + \mathbb{E}_{x} \left[ L_{2,1}\left(f;x \right)  \right]  \overline{L_{2,2}\left(Y \right)}+\widehat{\mathbb{C}ov}\left( L_{2,1}\left(f;X \right),  L_{2,2}\left(Y \right)  \right)    \Big\}.  
\end{align}



This approach can be used to transform any supervised ERM process into a semi-supervised one. However, in this work we demonstrate this idea on a common class of ERM models, which is the generalized-linear models (GLM), where we are looking for the best linear predictor $\beta^*$ that satisfies:

\begin{equation*}
    \beta^* = \underset{\beta\in \mathbb{R}^p}{argmin} \left\{ \mathbb{E}_{x_0y_0}\left[ L\left(\beta;x_0,y_0 \right) \right]   \right\},
\end{equation*}
where the loss function $L$ can be written as $L\left(\beta;x_0,y_0 \right) = G(x_0^T\beta)- x_0^T\beta y_0$, with $G'=g$, for some known monotone increasing \textit{link function} $g:\mathbb{R}\to \mathbb{R}$. This learning procedure is usually applied under the assumption that $\mathbb{E}[y|x]= g(x^T\beta)$, for some $\beta \in \mathbb{R}^p$. In the GLM setting, the supervised ERM suggests to solve the following optimization problem:
\begin{align}\label{e1}
\hat{\beta} &= \underset{\beta\in \mathbb{R}^p}{argmin} \left\{ \hat{L}(\beta,X,Y) \right\} = \underset{\beta\in \mathbb{R}^p}{argmin} \left\{\frac{1}{n}\sum_{i=1}^n G(x_i^T\beta)-x_i^T\beta y_i \right\}. 
\end{align}
This fitting procedure covers linear and logistic regression, among many others. We plug-in the loss function $L\left(\beta;x_0,y_0 \right)$ into (\ref{gen_tilde}) and (\ref{gen_Breve}), to define our semi-supervised estimators of $\beta$: \begin{align}\label{e2}
     \tilde{\beta} &= \underset{\beta\in \mathbb{R}^p}{argmin} \left\{  \tilde{L}(\beta,X,Y) \right\} =\underset{\beta\in \mathbb{R}^p}{argmin} \left\{ \mathbb{E}_x \left[G(x^T\beta)\right]- \frac{1}{n}\sum_{i=1}^n x_i^T\beta y_i \right\}, \\ \label{e3}
     \Breve{\beta} &= \underset{\beta\in \mathbb{R}^p}{argmin} \left\{  \Breve{L}(\beta,X,Y) \right\}=\underset{\beta\in \mathbb{R}^p}{argmin} \left\{ \mathbb{E}_x\left[G(x^T\beta)\right]- \mathbb{E}_{x}  \left[ x^T\beta\right] \overline{Y} - \widehat{\mathbb{C}ov}(X\beta,Y) \right\},
\end{align}
where $\widehat{\mathbb{C}ov}(X\beta,Y)  = \sum_{i=1}^n\left( x_i^T\beta-  \overline{X\beta} \right) \left(  y_i - \overline{Y} \right)/n$.


Throughout this work we study and compare between the prediction performance of the three estimators, $\hat{\beta}$, $\tilde{\beta}$, and $\Breve{\beta}$, as well as the null model, which we set to be the model that uses the known population mean, $\mu_0=\mathbb{E}_{x}\left[g(x^T\beta)\right]$, for prediction over any new point $x_0$. In accordance to previous works, we observed that in high-variance or high-dimension situations, the semi-supervised models are superior to the standard supervised model. However, in these situations both approaches may deliver inferior prediction performance. We choose to capture this effect by comparing them to the null model and challenging the semi-supervised approach to do better than the supervised, in situations where the predictions are actually useful. We consider our suggested semi-supervised estimator to be \textit{effective} if it performs better than both the standard estimator $\hat{\beta}$ and the null model. For the sake of practicality, we define the \textit{adaptive empirical estimator}, $\beta^D$, to be the model that chooses $\Breve{\beta}$ only when it is assumed to be effective, according to some data-based estimations. This will be clear in the sequel, as we analyze the prediction performance of the above estimators.

The idea of utilizing unlabeled data to improve prediction might seem contrary to the conditionality principle, which states that inference over $\beta$ should be conditioned on $X$, since $X$ is ancillary. However the goal of the learning process we discuss here is prediction over new random point $x_0$, rather than inference over the real vector $\beta$. It was shown by \cite{brown1990ancillarity}, Remark 2.1.3, that under linear model with Gassian noise, the prediction rule $x_0^T\hat{\beta}$ is dominated (in terms of squared loss) by the rule $x_0^T\hat{\beta}\left(1-\hat{\rho} \right)$, where $\hat{\rho}$ is a function of the MLE $\hat{\beta}$, the known covariates $X$, \textbf{and also} $\mathbb{E}_{x_0}[x_0x_0^T]$. The main result of that paper was that the estimation of the intercept term in linear regression, $\hat{\alpha}$, is dominated by another estimator of the form $\hat{\alpha}\left(1-\hat{\rho} \right)$, with $\hat{\rho}$ being more complex and using information about the distribution of $X$. The conclusion was that the least squared estimator $\hat{\alpha}$ is inadmissible. \cite{brown1990ancillarity} conclude with some remarks about ancillary statistics, distinguishing between statistical inference and point estimation (i.e. prediction), stating that point estimation can be improved in terms of integrated (unconditional) loss, by taking into account the distribution of ancillary statistic. 

In Section \ref{Semi-supervised OLS}, we focus on the special case of OLS model where the link function is the identity function. For a true linear model, we establish the "Sandwich" phenomenon for Gaussian covariates which states that the previously suggested semi-supervised estimator $\tilde{\beta}$ can not improve both on the standard OLS estimator and the null model simultaneously. However, our new-suggested estimator $\Breve{\beta}$ is capable of slightly improving on both competitors simultaneously, for an explicit range of signal-noise combinations. This result is further generalized to a wide class of distributions under asymptotic setting. In Section \ref{Uns_est} we extend the discussion to general distribution for the covariates and present the main methodology for identifying the usefulness of the unlabeled data according to the learning model and some prior assumptions. By simulations, we provide evidence that this methodology is indeed practical and accurate. Moreover, we find that for Uniform covariates, the semi-supervised estimators are effective for a wide range of signal-noise combinations. The scenario of misspecified linear regression is discussed in Section \ref{Accounting for bias}, concluding that the SSL may achieve a substantial improvement in some settings of mis-specification.

In Section \ref{Semi-supervised GLS-ERM}, we analyze the semi-supervised GLM-ERM procedure for general link function. We first show that the suggested optimization problems can be solved by gradient descent algorithm in a semi-supervised fashion. In Section \ref{Theoretical analysis}, we show that by using a quadratic approximation we can implement the same methodology and achieve approximate insights about the usefulness of the unlabeled data as for OLS. We then present in Section \ref{GLM_Simulations} results of extensive simulations that support the theoretical analysis. Further possible applications of the semi-supervised ERM approach are discussed in Section \ref{Discussion}.

\section{Semi-Supervised OLS} \label{Semi-supervised OLS}

\subsection{Distributional assumptions and variance analysis}

The OLS model is a special case of GLM-ERM where the link function is the identity function, and the loss function $L$ can be written as: $L\left(\beta;x,y \right) = (x^T\beta)^2/2- x^T\beta y$. Under the innocuous assumption of exchangeability  between expected value and derivative, the linear predictors from (\ref{e1}), (\ref{e2}), and (\ref{e3}) can be explicitly written:
\begin{equation*}
    \hat{\beta} = (X^TX)^{-1}X^TY \hspace{2mm};\hspace{2mm} \tilde{\beta} = H^{-1}X^TY  \hspace{2mm};\hspace{2mm} \Breve{\beta} = H^{-1}n\left(\mathbb{E}[X]\cdot \overline{Y}+\widehat{\mathbb{C}ov}(X,Y)  \right),
\end{equation*}
where $H= \mathbb{E}[X^TX]$, $\widehat{\mathbb{C}ov}(X,Y) \in \mathbb{R}^p$, and $\left[\widehat{\mathbb{C}ov}(X,Y)\right]_j = \widehat{\mathbb{C}ov}\left( [X]_j,Y \right)$.

The above formulas exploit the crucial requirement that $(X^TX)^{-1}$ exists in order to define $\hat{\beta}$, but not $ \tilde{\beta}$ and $\Breve{\beta}$. Since we are interested in comparing between the supervised estimator and the semi-supervised ones in terms of mean error (over all possible training samples), we shall take the following distributional
assumption:
\begin{Assump}\label{A1} The distribution $P_x$ satisfies that $(X^TX)^{-1}$ exists with probability 1.
\end{Assump} 
We note that this assumption refers only to the validity of the standard supervised
estimator and not the semi-supervised ones. We also note that it holds for any continuous
distribution $P_x$. Since Assumption \ref{A1} is crucial for any analysis of $\hat{\beta}$, it is taken as granted
throughout this section.

The estimator $\tilde{\beta}$ was presented by \cite{tarpey2014paradoxical} and discussed by \cite{cook2015tarpey}, and \cite{christensen2015covariance} (denoted  there also by $\tilde{\beta}$). In these works, they compare between $Var(\hat{\beta})$ and $Var(\tilde{\beta})$, and present a condition for $Var(\hat{\beta})-Var(\tilde{\beta})$ to be positive definite and therefore $\tilde{\beta}$ to be the better estimator under Gaussian covariates assumption. As far as we know, the explicit expression for $\Breve{\beta}$ is a novel one even when considering the special case of OLS.

The prediction performance of any linear predictor $\dot{\beta}$ generated based on a random training sample $T$, can be summarized into a bias-variance decomposition  adopted by \cite{rosset2018fixed}: 
\begin{align*}
    R(\dot{\beta}) 
     =&  \frac{1}{2}\Bigg\{ \mathbb{E}_{X,x_0}\left( \mathbb{E}\left[x_0^T\dot{\beta}|X,x_0  \right] -f(x_0) \right)^2 +  \mathbb{E}_{X,x_0}\left[ \mathbb{V}ar\left(x_0^T\dot{\beta}|X,x_0  \right) \right]   -\mathbb{E}_{x_0}\left[ (f(x_0))^2 \right]  \Bigg\} \nonumber \\
    =& \frac{1}{2}\Big\{B(\dot{\beta})+V(\dot{\beta})-C \Big\}. 
\end{align*}
Here $B$ is the squared \textit{bias} and $V$ is the \textit{variance} term associated with the model building procedure. The term $C$ is a constant that does not depend on the learning procedure. In this section, we use the decomposition above in order to compare between different models.

The comparison between the supervised estimator and the semi-supervised ones is relevant only if $R(\hat{\beta})$ is well defined, which requires that $V(\hat{\beta})$ is well defined. Regardless of the true model $f(x)$, the variance term $V(\hat{\beta})$ can be written as follows:
\begin{align*}
    V(\hat{\beta})&=\frac{1}{n}\text{tr}\left( \mathbb{E}_{X}\left[ (X^TX)^{-1}X^TI_n\sigma^2 X(X^TX)^{-1}   \right] H\right) =\frac{\sigma^2}{n}\text{tr}\left(\mathbb{E}_{X}\left[ (X^TX)^{-1}\right] H\right),
    \end{align*}
which is only defined if $Q=\mathbb{E}_{X}\left[ (X^TX)^{-1}\right]$ exists. We note that under Assumption \ref{A1}, the matrix $X^TX$ is invertible with probability $1$, and therefore its eigenvalues are all positive with probability $1$, and can be written as $0<\lambda_1 \leq ... \leq\lambda_p$. Moreover, the eigenvalues of $(X^TX)^{-1}$ are simply  $0<1/\lambda_p \leq ... \leq1/\lambda_1$. This leads us to the following distributional assumption:

\begin{Assump}\label{A2} The distribution $P_x$ satisfies that $ \mathbb{E}_{X}\left[  1/ \lambda_1  \right] $ is finite.
\end{Assump}
This is not a mild assumption, and in particular it subsumes Assumption \ref{A1}. However, as we show in the following Proposition, this assumption is necessary for $R(\hat{\beta})$ to be finite.  
\begin{prop}\label{P0}
The variance term $V(\hat{\beta})$ is either well defined (Assumption \ref{A2} holds), or  $V(\hat{\beta})=\infty$ (Assumption \ref{A1} holds but Assumption \ref{A2} is violated), or $\hat{\beta}$ does not exist with positive probability and therefore $R(\hat{\beta})$ is undefined.  
\end{prop}

\begin{proof} 
Under Assumption \ref{A1}, we can write the eigenvalues of $(X^TX)^{-1}$ as $0<1/\lambda_p \leq ... \leq1/\lambda_1$. We note that $H$ is P.S.D and invertible, and we can write its eigenvalues as $0<h_1 \leq ... \leq h_p$. Using the main result from \cite{bushell1990trace}, we can show that:
\[\mathbb{E}_{X}\left[ \text{tr}\left(  H (X^TX)^{-1} \right) \right] \geq \mathbb{E}_{X}\left[ \sum_j h_j\frac{1}{\lambda_j}   \right] \geq  h_1 \mathbb{E}_{X}\left[ 1/ \lambda_1 \right]  \]
On the other hand, using standard properties, we can show that:
\[\mathbb{E}_{X}\left[ \text{tr}\left(  H (X^TX)^{-1} \right) \right] \leq \text{tr}\left(  H \right) \mathbb{E}_{X}\left[ \sum_j\frac{1}{\lambda_j}   \right] \leq  \text{tr}\left(  H \right) p \mathbb{E}_{X}\left[  1/ \lambda_1  \right], \]
and we ca write:
\[  h_1 \mathbb{E}_{X}\left[ 1/ \lambda_1 \right] \leq \mathbb{E}_{X}\left[ \text{tr}\left(  H (X^TX)^{-1} \right) \right] \leq  \text{tr}\left(  H \right) p \mathbb{E}_{X}\left[  1/ \lambda_1  \right].\]

Thus, if $\mathbb{E}_{X}\left[ 1/ \lambda_1\right]<\infty$, $\text{tr}\left(  HQ \right) = \mathbb{E}_{X}\left[ \text{tr}\left(  H (X^TX)^{-1} \right) \right]$ is bounded, and $V(\hat{\beta})$ is well defined. If $\mathbb{E}_{X}\left[ 1/ \lambda_1\right]=\infty$, we get that $\text{tr}\left(  HQ \right)=\infty$ and also $V(\hat{\beta})=\infty$. If $(X^TX)^{-1}$ does not exist with positive probability, so is $\hat{\beta}$. In this scenario, any expectation involving a function of $\hat{\beta}$ over $X$ is undefined.  
\end{proof}

Analyzing the variance terms of the semi-supervised estimators, we find the they are finite, regardless of any assumption over $P_x$, and can be written as follows:
\begin{align}\label{V_tilde}
    V(\tilde{\beta}) &= \frac{1}{n}\text{tr}\left( \mathbb{E}_{X}\left[ \mathbb{V}ar\left(\tilde{\beta}|X  \right) \right] H \right) 
    =\frac{\sigma^2}{n}\text{tr}\left(H^{-1} \mathbb{E}_{X}\left[ X^TX\right]\right) = \frac{\sigma^2p}{n}, \\ \label{V_breve}
    V(\Breve{\beta})&=\frac{1}{n}\text{tr}\left( H^{-1} \mathbb{E}_{X}\left[ \mathbb{V}ar\left(X^TY-n\overline{X}\cdot\overline{Y}|X  \right) \right] \right) =\left(1-\frac{1}{n} \right) \frac{\sigma^2p}{n}.
    \end{align}
From the above formulas and Proposition \ref{P0}, we conclude that that if Assumption \ref{A2} does not hold, $R(\hat{\beta})$ is either unbounded or undefined, and therefore the SSL is trivially beneficial in terms of mean prediction performance. However, in this work we seek to focus on the common scenario where $R(\hat{\beta})$ is well defined, and compare it with the semi-supervised alternatives. Therefore, Assumption \ref{A2} is taken as granted throughout this Section.

Under Assumption \ref{A2}, the result by \cite{groves1969note} implies that the matrix $\mathbb{E}\left[(X^TX)^{-1} \right] -\left( \mathbb{E}[X^TX]\right)^{-1}$ is positive semi-definite. Thus we have: $ \text{tr}\left(QH \right) \geq \text{tr}\left(H^{-1}H \right) = p$, and we conclude that the variance term is guaranteed to decrease with the use of unlabeled data. Further more, as $\sigma^2$ increases, the benefit of using unlabeled data increases as well. In the sequel, we analyze the squared bias term according to assumed true model $f(x)$, and carry out a dedicated comparison between all the estimators under discussion.




\subsection{True linear model} \label{Unb_OLS}
In this scenario we assume that the linear model is correct, meaning that $f(x)=x^T\beta$ for some $\beta \in \mathbb{R}^p$. Under this assumption, we can see that $\mathbb{E}[\hat{\beta} | X] = \beta$, which means that $\hat{\beta}$ is an unbiased estimator of $\beta$ for any covariate matrix $X$. On the other hand, for the estimators $\tilde{\beta}$ and $\Breve{\beta}$, only the unconditional expected value is equal to $\beta$:
\begin{align*}
    \mathbb{E}[\tilde{\beta}]&= \mathbb{E}_X\left[\mathbb{E}[\tilde{\beta}|X]\right]= \mathbb{E}_X\left[H^{-1}\left(X^T\mathbb{E}[Y|X]\right)   \right] =H^{-1} \mathbb{E}_X[X^TX]\beta  =\beta,
\end{align*}
and in the same manner we can show that $\mathbb{E}[\Breve{\beta}]= \beta$. We note that $\mathbb{E}[\tilde{\beta}|X]$ and $\mathbb{E}[\Breve{\beta}|X]$ may be different from $\beta$ as $H^{-1}X^TX$ may be different from $I_p$. Focusing on the squared bias term, we can see that $B(\dot{\beta})=0$ for every estimator $\dot{\beta}$ that satisfies $\mathbb{E}[\dot{\beta}|X]=\beta$, since $\mathbb{E}[x_0^T\dot{\beta}|X,x_0]=x_0^T\beta=f(x_0)$. However, if $\mathbb{E}[\dot{\beta}]=\beta$ then:
\begin{align*}
    B(\dot{\beta}) &=\frac{1}{n}\text{tr}\left(\mathbb{E}_X\left[(\mathbb{E}[\dot{\beta}|X]-\beta)(\mathbb{E}[\dot{\beta}|X]-\beta)^T    \right] H \right) = \frac{1}{n}\text{tr}\left(\mathbb{V}ar_X\left(\mathbb{E}[\dot{\beta}|X]\right) H \right)
\end{align*}
We can place $\tilde{\beta}$ and $\Breve{\beta}$ instead of $\dot{\beta}$ to get the bias terms as follows: 
\begin{align} \label{B_tilde_expr}
    B(\tilde{\beta}) &= \frac{1}{n} \text{tr}\left( H^{-1}\mathbb{V}ar_X(X^TX\beta) \right),\\
    B(\Breve{\beta}) &= \frac{1}{n} \text{tr}\left( H^{-1}\mathbb{V}ar_X\left( n\mathbb{E}[X]\overline{X\beta} + n\widehat{ \mathbb{C}ov}(X,X\beta) \right) \right), \nonumber
\end{align}
where: $\overline{X\beta}= \sum_{i=1}^n x_i^T\beta/n$, $\left[\mathbb{E}[X] \right]_j= \mathbb{E}[x_j]$, and $\left[\widehat{\mathbb{C}ov}(X,X\beta)\right]_j = \widehat{\mathbb{C}ov}\left( [X]_j,X\beta \right)$.

We conclude that using the unlabeled data in this scenario induces some bias to the prediction error. We get a \textit{bias-variance trade-off} between the supervised and the semi-supervised approaches. The increase in the bias term of the semi-supervised estimators does not depend on $\sigma^2$. Therefore, for any distribution of $X$ and real vector $\beta$, there is a threshold $\tilde{F}(\beta)$, where higher values of $\sigma^2$ will make the semi-supervised model superior to the regular OLS model. In particular, we can write:
\begin{align} \label{F_tilde}
    R(\tilde{\beta})< R(\hat{\beta}) & 
    \iff \sigma^2 > \frac{\text{tr}\left(H^{-1} \mathbb{V}ar_X(X^TX\beta) \right)}{ \text{tr}\left(QH  \right)-p} := \tilde{F}(\beta).
\end{align}
In the same manner we can write the threshold $\Breve{F}(\beta)$, where the estimator $\Breve{\beta}$ becomes better than the OLS estimator $\hat{\beta}$: 
\begin{align}\label{F_Breve}
    R(\Breve{\beta})< R(\hat{\beta})  
    &\iff   \sigma^2 > \frac{\text{tr}\left(H^{-1} \mathbb{V}ar_X\left( \mathbb{E}[X]\overline{X\beta} + \widehat{ \mathbb{C}ov}(X,X\beta) \right) \right)}{ \text{tr}\left(QH  \right)-p \frac{n-1}{n}} := \Breve{F}(\beta).
\end{align}

The general conclusion that the semi-supervised process is guaranteed to be better for sufficiently noisy data, coincides with that of \cite{tarpey2014paradoxical}. Under the assumption of Gaussian covariates and by the properties of Wishart distribution, an explicit inequality is presented by \cite{christensen2015covariance} in terms of $n$, $p$, and the signal-noise combination. 

However, as discussed in the introduction, for high enough value of $\sigma^2$, the  null model may be better than both models, and the learning process is actually not useful. In this particular case,  the null model predicts the value $0$ for every point $x_0$. It is clear that the variance term of the null model is zero, and the squared bias term can be written as $B(0) =  \mathbb{E}_{x_0}\left(x_0^T\beta \right)^2 =\beta^TH\beta /n$. We can see that the error associated with the null model does not depend on $\sigma^2$, and therefore there is a threshold $\tilde{U}(\beta)$, where lower values of $\sigma^2$ make the semi-supervised model superior to the null model. In particular, we can write:
\begin{align} \label{U_tilde}
    R(\tilde{\beta})< R(0) 
     \iff & \sigma^2 < \frac{1}{p} \left(   \beta^T H  \beta  - \text{tr}\left(\mathbb{V}ar_X(X^TX\beta)H^{-1}\right)  \right) := \tilde{U}(\beta), \\ \label{U_Breve}
     R(\Breve{\beta})< R(0)  \iff &  
       \sigma^2 < \frac{n}{p(n-1)} \left(   \beta^T H  \beta  - nB(\Breve{\beta}) \right) := \Breve{U}(\beta).
\end{align}

We conclude that $\tilde{\beta}$ is effective when $ \tilde{F}(\beta) < \sigma^2 <\tilde{U}(\beta)$, and $\Breve{\beta}$ is effective when $ \Breve{F}(\beta) < \sigma^2 <\Breve{U}(\beta)$. As we show next, when the covariates are Gaussian, $\tilde{F}(\beta) = \tilde{U}(\beta)$ for every vector $\beta$, which means that there is no value of $\sigma^2$ for which the estimator $\tilde{\beta}$ is effective. We call it the \textit{Sandwich phenomenon} because of the fact that $R(\tilde{\beta})$ is always between $R(\hat{\beta})$ and $R(0)$. On the other hand, we find that $R(\Breve{\beta})$ is smaller than $R(\tilde{\beta})$ with $O(1/n)$ difference, and in accordance $\Breve{F}(\beta) < \Breve{U}(\beta)$. We conclude that $\Breve{\beta}$ can achieve an improvement both on the standard OLS model and the null model when $\sigma^2$ is within the interval $[\Breve{F}(\beta), \Breve{U}(\beta)]$.   

\begin{Thm} \label{Normal_Covariates_Thm}
Assuming true linear model and Gaussian distribution for the covariates, the semi-supervised estimator $\tilde{\beta}$ can not improve both on the standard OLS model and the null model simultaneously. The semi-supervised estimator $\Breve{\beta}$ uniformly satisfies that $V(\Breve{\beta})/V(\tilde{\beta})=1-1/n$, and $B(\Breve{\beta})/ B(\tilde{\beta}) \in \left[1-2/n,1-1/n\right]$, and therefore $R(\Breve{\beta})<R(\tilde{\beta})$. Consequently, $\Breve{F}(\beta) < \Breve{U}(\beta)$  and there is a guaranteed range of $\sigma^2$ where $\Breve{\beta}$ improves both on the standard OLS model and the null model with an $O(1/n)$ term. 

\end{Thm}

\begin{proof}

Assuming that $x\sim MN(\textbf{0},\Sigma)$, and therefore $X^TX\sim W_p(n,\Sigma)$, it was shown by \cite{christensen2015covariance} that $\mathbb{V}ar_X(X^TX\beta) = n\left[  \beta^T\Sigma\beta\Sigma+\Sigma\beta \beta^T \Sigma \right]$. Moreover, we can write: $H=n\Sigma$ and $Q=\Sigma^{-1}/(n-p-1)$. Putting it back in (\ref{F_tilde}), we can explicitly write the lower threshold value $\tilde{F}(\beta)$ as follows:
\begin{align*}
    \tilde{F}(\beta) &= \frac{\text{tr}\left(\frac{1}{n}\Sigma^{-1}\cdot n \left[ \beta^T\Sigma\beta\Sigma+\Sigma\beta \beta^T \Sigma \right]  \right)}{\text{tr}\left( n\Sigma \Sigma^{-1}\frac{1}{n-p-1} \right) -p } \nonumber 
    = \beta^T\Sigma\beta \frac{n-p-1}{p}  := \tilde{t}(\beta)
\end{align*}
On the other hand, the upper threshold $\tilde{U}(\beta)$ is also equal to  $\tilde{t}(\beta)$:
\begin{align*}
     \tilde{U}(\beta) &= \frac{1}{p}\left[n\cdot \beta^T\Sigma\beta - \text{tr}\left(\frac{1}{n}\Sigma^{-1}\cdot n \left[ \beta^T\Sigma\beta\Sigma+\Sigma\beta \beta^T \Sigma \right]  \right) \right] = \beta^T\Sigma\beta \frac{n-p-1}{p}
\end{align*}


As for $\Breve{\beta}$, it follows immediately from (\ref{V_tilde}) and  (\ref{V_breve}) that $V(\Breve{\beta}) = (1-1/n) V(\tilde{\beta})$, and we are showing similar relationship between $B(\Breve{\beta})$ and $B(\tilde{\beta})$. Denote $J=\textbf{1}_n\textbf{1}_n^T-I_n$, we can write $X^T\textbf{1}_n\textbf{1}_n^TX\beta = X^TX\beta+ X^TJX\beta$, where these two terms are uncorrelated. Moreover, we can show that $\mathbb{V}ar_X(X^TJX\beta) = n(n-1) \beta^T\Sigma\beta\Sigma$.

Using the properties above, and the fact that $\mathbb{E}[X]=\textbf{0}$, we can write $B(\Breve{\beta})$ as follows:
\begin{align}\label{l44}
    B(\Breve{\beta}) &=  \frac{1}{n} \text{tr}\left( H^{-1}\mathbb{V}ar_X\left(X^TX\beta - \frac{1}{n}X^T\textbf{1}_n\textbf{1}_n^TX\beta \right) \right)\\
     &=  \frac{\frac{n-1}{n}\left(p+  \frac{n-1}{n}\right)}{n}\beta^T\Sigma\beta = \left(1-\frac{1}{n}\right) \left(1-\frac{1}{n(p+1)}\right)B(\tilde{\beta})  . \nonumber
\end{align}
We conclude that $ (1-2/n) B(\tilde{\beta})<B(\Breve{\beta}) < (1-1/n) B(\tilde{\beta})$ for every value of $\sigma^2$, and we can use the derivations in (\ref{F_Breve}) and (\ref{U_Breve}) to show that:
\begin{align*}
    \Breve{F}(\beta) < \left(1-\frac{1}{n}\right) \left(1-\frac{1}{n(p+1)}\right)\tilde{t}(\beta)  \hspace{3mm}; \hspace{3mm}
    \Breve{U}(\beta) > \left(1+\frac{1}{n-1}\right)\tilde{t}(\beta).
\end{align*}
When $\sigma^2 \in \left[\Breve{F}(\beta), \Breve{U}(\beta) \right]$, $\Breve{\beta}$ outperforms both the null model and the standard OLS estimator. The maximum difference between $\Breve{\beta}$ and the second best estimator occurs when $\sigma^2 = \tilde{t}(\beta)$, where $ R(\tilde{\beta}) = R(0) =  R(\hat{\beta})$, and is an $O(1/n)$ term. \end{proof}

Theorem \ref{Normal_Covariates_Thm} provides a theoretical analysis for the case of Gaussian covariates, resulting in explicit expressions for the out-of-sample loss of the three estimators under discussion and the null model. Consider an asymptotic setup where $p/n \to \gamma\in(0,1)$ and $\beta^T\Sigma \beta \to \tau^2$ as $n\to \infty$, then for the OLS model with Gaussian covariates, we can simply write: 
\begin{equation*}
    V(\hat{\beta}) \to \frac{\gamma \sigma^2 }{1-\gamma} \hspace{2mm};\hspace{2mm} V(\tilde{\beta})=\gamma \sigma^2  \hspace{2mm};\hspace{2mm} B(\tilde{\beta})\to \gamma \tau^2  \hspace{2mm};\hspace{2mm} B(0) \to \tau^2.
\end{equation*}

Let us now study the above quantities under the following mechanism for generating the covariates $x$: We draw $\tilde{x}\in \mathbb{R}^p$ , having iid components $\tilde{x}_i \sim F$, $i=1,\cdots,p$, where $F$ is any distribution with zero mean, unit variance, and a finite fourth moment $q$. We then set $x=\Sigma^{1/2}\tilde{x}$, where $\Sigma \in  \mathbb{R}^{p\times p}$ is symmetric positive definite, and its smallest eigenvalue is bounded from zero for all $p$. It is easy to show (for example, see the proof of \cite{hastie2019surprises}, Proposition 2), that under this generating mechanism, with mild moment assumptions, the conditional variance term, $V_X(\hat{\beta})=\sigma^2\text{tr}(\Sigma(X^TX)^{-1})$, converges almost surly in $X$ to $\sigma^2\gamma/(1-\gamma)$ as $n,p\to \infty$. Although we cannot rigorously argue that $V(\hat{\beta})= \mathbb{E}_X[V_X(\hat{\beta})]$ converges to the same expression, we may refer it as the ``common behavior'' of the variance term under this setting.

We now show that all other terms are asymptotically equal to the Gaussian case, under this general mechanism. It is clear that $V(\tilde{\beta}), V(\breve{\beta})$ are unchanged since they are independent of the distribution of the covariates. Moreover, the bias term $B(0)$ can be written as follows. 
\begin{equation*}
     B(0) = \frac{1}{n}\beta^TH\beta =   \frac{1}{n}\beta^T \mathbb{E}\left[ \Sigma^{1/2}\tilde{X}^T\tilde{X}\Sigma^{1/2}  \right] \beta  =  \frac{1}{n}\beta^T \Sigma^{1/2} nI_p \Sigma^{1/2}  \beta =  \beta^T\Sigma \beta \to \tau^2 .
\end{equation*}
The next proposition deals with the asymptotic calculation of $B(\tilde{\beta})$.

\begin{prop} \label{P1}
Assume that $x$ is generated as above, then for the OLS model: $ B(\tilde{\beta}) \to \gamma \tau^2$.
\end{prop}
\begin{proof}
Using the derivation of $B(\tilde{\beta})$ from (\ref{B_tilde_expr}) and the properties of the distribution $F$, we can show that:  
\begin{align*}
     B(\tilde{\beta})&= \frac{1}{n} \text{tr}\left( H^{-1}\mathbb{V}ar_X(X^TX\beta) \right) =   \frac{1}{n} \text{tr}\left( \frac{1}{n} \Sigma^{-1} \mathbb{V}ar_X(X^TX\beta) \right) \\
     &=   \frac{1}{n^2} \text{tr}\left(   \left[ \left( n(p+q-1)+n(n-1) \right)I_p-n^2I_p  \right]  \Sigma^{1/2}\beta \beta^T \Sigma^{1/2}  \right) = \frac{p+q-2}{n}\beta^T\Sigma \beta .
\end{align*}
We use the fact that $(p+q-2)/n\to \gamma$, as $n,p\to \infty$, to find that $ B(\tilde{\beta}) \to \gamma \tau^2$.\end{proof}

Summarizing Proposition \ref{P1} with previous results, we conclude that the Sandwich phenomenon is the common behavior over a wide class of distributions under the asymptotic setting, with small violation in favor of $\tilde{\beta}$ when $q<3$, and against $\tilde{\beta}$ when $q>3$. However, the requirement that the covariate vector $x$ be expressible as $\Sigma^{1/2}\tilde{x}$ does limit the set of covariate joint distributions for which this result applies (see \cite{rosset2018fixed} for discussion). As we show next, in other scenarios such as distributions that violate this generating mechanism, the estimators $\tilde{\beta}$ may deliver a substantial improvement or deterioration compared to both non-SSL competitors simultaneously.

\subsection{Unsupervised thresholds estimation}\label{Uns_est}
We now extend the discussion to general distribution of the covariates. We point to the fact that for any given vector $\beta$, the threshold values can be estimated using the set of unlabeled data $Z$. For example, we can (precisely, for large enough $m$) estimate $H$ by $(n/m)\sum_{i=1}^m z_iz_i^T$, and  $ \mathbb{E}_X\left[X \right]$ by $(1/m)\sum_{i=1}^m  z_i$. Estimating $Q$ and $\mathbb{V}ar_X(X^TX\beta)$ 
can be done by sampling large amount of covariate matrices $X$ from the set $Z$ and computing the statistics from that sample. We can then derive the threshold values for $\tilde{\beta}$ and $\Breve{\beta}$ by the formulas presented earlier, and determine the usefulness of the semi-supervised learning, free from assumptions over $P_x$. In general, the suggested methodology can be described as follows:
\begin{enumerate}
    \item Derive the formulas for the threshold values according to the learning model and the assumptions on the true model.
    \item Approximately calculate the threshold values by using the unlabeled data on hand.   
    \item Identify the usefulness of the SSL according to the threshold values and some prior knowledge or estimation of the signal-noise combination.
\end{enumerate}

We confirm below by simulation, that this methodology is indeed practical and accurate in many scenarios. We find that unlike the Gaussian covariates case, when the covariates are from Uniform distribution, $\tilde{\beta}$ is effective for some $(\sigma^2,\beta)$ combinations and a substantial improvement may be achieved by applying our SSL approach. Moreover, we show that the adaptive estimator $\beta^D$, delivers uniform improvement over the supervised estimator. The adaptive model chooses between $\hat{\beta}$ and $\Breve{\beta}$ according to estimates of $\sigma^2$ and $\Breve{F}(\beta)$. We estimate $\sigma^2$ by the standard unbiased estimator $\hat{\sigma}^2=RSS(\hat{\beta})/(n-p)$, and $\Breve{F}(\beta)$ by the dedicated formula evaluated at $\Breve{\beta}$, with bias correction to the estimation of $\mathbb{V}ar_X(X^TX\beta)$.


We now demonstrate the unsupervised estimation methodology in another scenario, where the vector $\beta$ is also random. Let us assume that $\beta$ is drawn from prior distribution $P_\beta$ such that $ \mathbb{E}\left[ \beta \right] = \textbf{0}$ and $ \mathbb{E}\left[ \beta \beta^T \right] = \tau^2 I_p$. In this case, the mean out-of-sample prediction error is taken over all possible $\beta$'s. Therefore, the bias terms $B(\tilde{\beta})$ and $B(0)$ can be written as follows:
\begin{align*}
    B(\tilde{\beta}) &= \frac{1}{n} \mathbb{E}_\beta\left[ \text{tr}\left( H^{-1}\mathbb{V}ar_{X}(X^TX\beta) \right) \right]  = \frac{\tau^2}{n} \text{tr}\left( H^{-1}\mathbb{E}_{X}\left[(X^TX)^2 \right] -H \right), \\
    B(0) & = \frac{1}{n}\mathbb{E}_\beta\left[\beta^TH\beta\right]  = \frac{\tau^2}{n} \text{tr}\left( H \right).
\end{align*}

Using the above results, we define the threshold values for the noise-to-signal ratio, $\sigma^2/\tau^2$, where the semi-supervised estimator is effective, as follows:
\begin{align*}
    R(\tilde{\beta})< R(\hat{\beta}) 
    & \iff \frac{\sigma^2}{\tau^2} > \frac{\text{tr}\left( H^{-1}\mathbb{E}_{X}\left[(X^TX)^2 \right] -H \right)}{ \text{tr}\left(QH  \right)-p} := \tilde{F},\\
    R(\tilde{\beta})< R(0) &\iff \frac{\sigma^2}{\tau^2} < \frac{1}{p}     \text{tr}\left(2H - H^{-1}\mathbb{E}_{X}\left[(X^TX)^2\right]\right)   := \tilde{U}.
\end{align*}

The threshold values $\tilde{F}$, and $\tilde{U}$ depend only on the distribution $P_x$ and can be evaluated using the unlabeled data set. We can then use these estimates to determine limits on the usefulness of the semi-supervised estimator $\tilde{\beta}$. In any case that $\sigma^2 / \tau^2 \in [\tilde{F},\tilde{U} ]$, the estimator $\tilde{\beta}$ is effective. If $\tilde{U}\leq \tilde{F}$, then there is no range of $\sigma^2 / \tau^2$ for which the estimator $\tilde{\beta}$ is effective. In the Gaussian covariates case, with $E\left[(X^TX)^2 \right] = n(n+1)\Sigma^2+n\Sigma \cdot \text{tr}(\Sigma)$, we find that the Sandwich phenomenon holds in this scenario as well:
\begin{equation}\label{t_tilde}
    \tilde{F}=\tilde{U}= \text{tr}(\Sigma)(n-p-1)/p : = \tilde{t},
\end{equation}


As for $\Breve{\beta}$, using the derivation in (\ref{l44}), we can show that:
\begin{align*}
     B(\Breve{\beta}) =& \frac{\tau^2}{n} \text{tr}\left( \frac{(n-1)^2}{n^2}\left[ H^{-1}\mathbb{E}_{X}\left[(X^TX)^2 \right] -H \right]\right)\\
     +& \frac{\tau^2}{n^3} \text{tr}\left(   H^{-1}\left[ \mathbb{E}_{X}\left[(X^TJX)^2 \right] - \left( \mathbb{E}_{X}\left[X^TJX \right] \right)^2  \right]  \right),
\end{align*}
and formulas for the threshold values $\Breve{F}$ and $\Breve{U}$ follow in the same manner as for $\tilde{F}$ and $\tilde{U}$. 

In this scenario, an estimate of the noise-to-signal ratio (NSR) is required in order to decide which of the estimators to use (and define $\beta^D$). We suggest the following estimator:
\[ \widehat{NSR}=\frac{\hat{\sigma}^2}{\hat{\tau}^2}  \hspace{2mm};\hspace{2mm} \hat{\tau}^2= max\left\{ \left(\frac{\sum y_i^2}{n} - \hat{\sigma}^2\right)/\text{tr}\left(\mathbb{E}\left[xx^T \right] \right)  , 0 \right\} \hspace{2mm};\hspace{2mm} \hat{\sigma}^2=\frac{RSS(\hat{\beta})}{n-p}.  \]
The estimator $\beta^D$ is equal to $\hat{\beta}$ if $ \widehat{NSR}<\Breve{F}$, equal to $\Breve{\beta}$ if $ \Breve{F}<\widehat{NSR}<\Breve{U}$, and equal to the null estimator otherwise. In any case where $\Breve{U}<\Breve{F}$, $\beta^D$ is equal to $\hat{\beta}$ if $ \widehat{NSR}<\text{tr}(H)/\text{tr}(HQ)$ and equal to the null estimator otherwise.

\subsection{Simulations for true linear model} \label{OLS_Und_Simulations}
We empirically study the predictive performance of the three OLS estimators under discussion, in different problem settings, by two-step simulations:
\begin{enumerate}
\item {\bf Unsupervised estimation.} Evaluating the threshold values according to the data generating mechanism, using a large fixed data-set $\mathcal{Z}$ of $M=5\cdot 10^4$ unlabeled observations of $p=25$ dimension.

\item {\bf Supervised simulation.} Generating $K=5000$ random training sets of $n=50$ labeled observations ($X,Y$) and $m=5000$ unlabeled observations ($Z$), for various values of $\sigma^2$. We fit the three estimators ($\hat{\beta},\tilde{\beta},\Breve{\beta}$) for each one of the training sets and calculate the mean prediction error over the data-set $\mathcal{Z}$. We also set $\beta^D$ according to the decision rule and store its prediction error. The outcome is four curves describing the Random-X prediction error (average over the $K$ samples) changing with $\sigma^2$ for each one of the estimators. For simplicity we denote by $R$, the reducible error $B+V$ (ignoring the fixed component $C$), throughout this subsection.         

\end{enumerate}
We perform experiments in a total of six data generating mechanisms, based on three different distributions for the covariate vector $x$, and two different generating mechanisms for the mean function $f(x)$. The three generating models for $x$ are as follows:

\begin{itemize}
\item \textit{Gaussian}. We choose $x \sim MN(\textbf{0}_p, \Sigma)$, where $\Sigma$ is block-diagonal, containing five blocks
such that all variables in a block have pairwise correlation $\rho=0.9$.
\item \textit{Uniform}. We define $x$ by taking Gaussian random vector as above, then applying the inverse Gaussian
distribution function componentwise. In other words, this can be seen as a Gaussian copula with Uniform marginals.
\item \textit{$t(8)$}. We define $x$ by taking Gaussian random vector as above, then adjust the marginal distributions appropriately to achieve Gaussian copula with $t(8)$ marginals.
\end{itemize}

The three distributions above are scaled to have zero mean and unit marginal variance. The marginal fourth moments are $3$, $1.8$, and $4.5$ respectively. The two generating models for the mean function $f(x) = \mathbb{E}[y|x]$ are as follows:

\begin{itemize}
\item \textit{Constant}-$\beta$. $f(x) = \beta \sum_{j=1}^p x_j \hspace{3mm},\hspace{3mm} \beta \in \{0.25,0.5,\cdots ,1.5 \}$. 
\item \textit{Random}-$\beta$. $f(x) = x^T \beta \hspace{3mm},\hspace{3mm} \beta \sim MN(\textbf{0},I_p) $.
\end{itemize}

In the constant-$\beta$ scenario (Figure \ref{f1}), the unsupervised estimation of the threshold-values is calculated for any value of $\beta$ according to the formulas in Section \ref{Unb_OLS}. We also present the value of $\tilde{t}(\beta=1.5)$ according to Theorem 1 by horizontal dashed black line. The supervised simulation is carried out for $\beta=1.5$ and a range of $\sigma^2$ that covers both $\Breve{F}(\beta=1.5)$ and $\Breve{U}(\beta=1.5)$. On the supervised simulation plots, we mark the constant value of the null risk by a horizontal dashed black line. The estimated threshold-values  $\Breve{F}(\beta=1.5)$ and $\Breve{U}(\beta=1.5)$ are marked by vertical dashed blue lines.

In the random-$\beta$ scenario (Figure \ref{f2}), the unsupervised estimation of the threshold-values is calculated according to the formulas in Section \ref{Uns_est}, referring to the ratio $\sigma^2/\tau^2$. The supervised simulation is carried out for a range of $\sigma^2$ that covers both $\Breve{F}$ and $\Breve{U}$, while $\tau^2$ is fixed at $1$. We mark the estimated threshold-values $\Breve{U}$ and $\Breve{F}$ by vertical blue dashed lines, as well as the value of $\tilde{t}$ according to (\ref{t_tilde}) by vertical black dashed line.

We can see good agreement between the unsupervised simulations and the prediction error curves in the supervised-simulations: the curve of $R(\Breve{\beta})$ intersects with $R(\hat{\beta})$ ($R(0)$) in the estimated point of $\Breve{F}$ ($\Breve{U}$). As expected, $R(\Breve{\beta})$ is lower than $R(\tilde{\beta})$ in all six scenarios. Moreover, we can see that the Gaussian covariates comply with the Sandwich phenomenon while the Uniform ($t(8)$) covariates have some range of $\sigma^2$ for which $\tilde{\beta}$ substantially outperforms (underperforms) both $\hat{\beta}$ and the null model. This demonstrates the role of the fourth moment regarding the effectiveness of the SSL. We can also see that $R(\beta^D)$ is lower than $R(\hat{\beta})$ in all six scenarios, over the entire range, and is the best estimator in the random-$\beta$ scenario.

\begin{figure}[h!]  
\centering
\includegraphics[width=15cm]{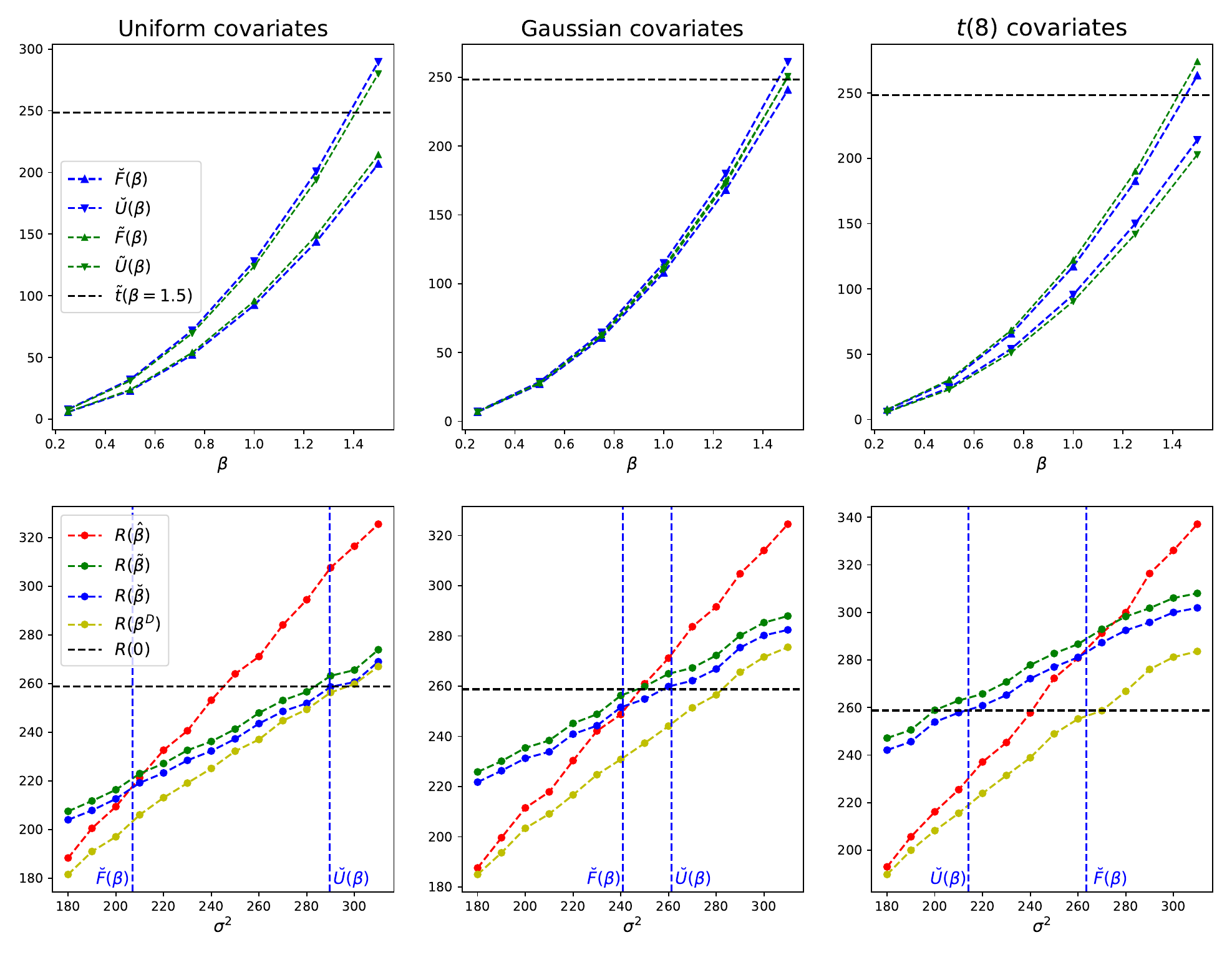}
\caption{Results of simulations for true linear model with constant-$\beta$ link, for Gaussian (middle), Uniform (left), and $t(8)$ covariates (right). The unsupervised estimates are presented at the top, and supervised simulations at the bottom.}
\label{f1} 
\end{figure}

\begin{figure}[h!]  
\centering
\includegraphics[width=15cm]{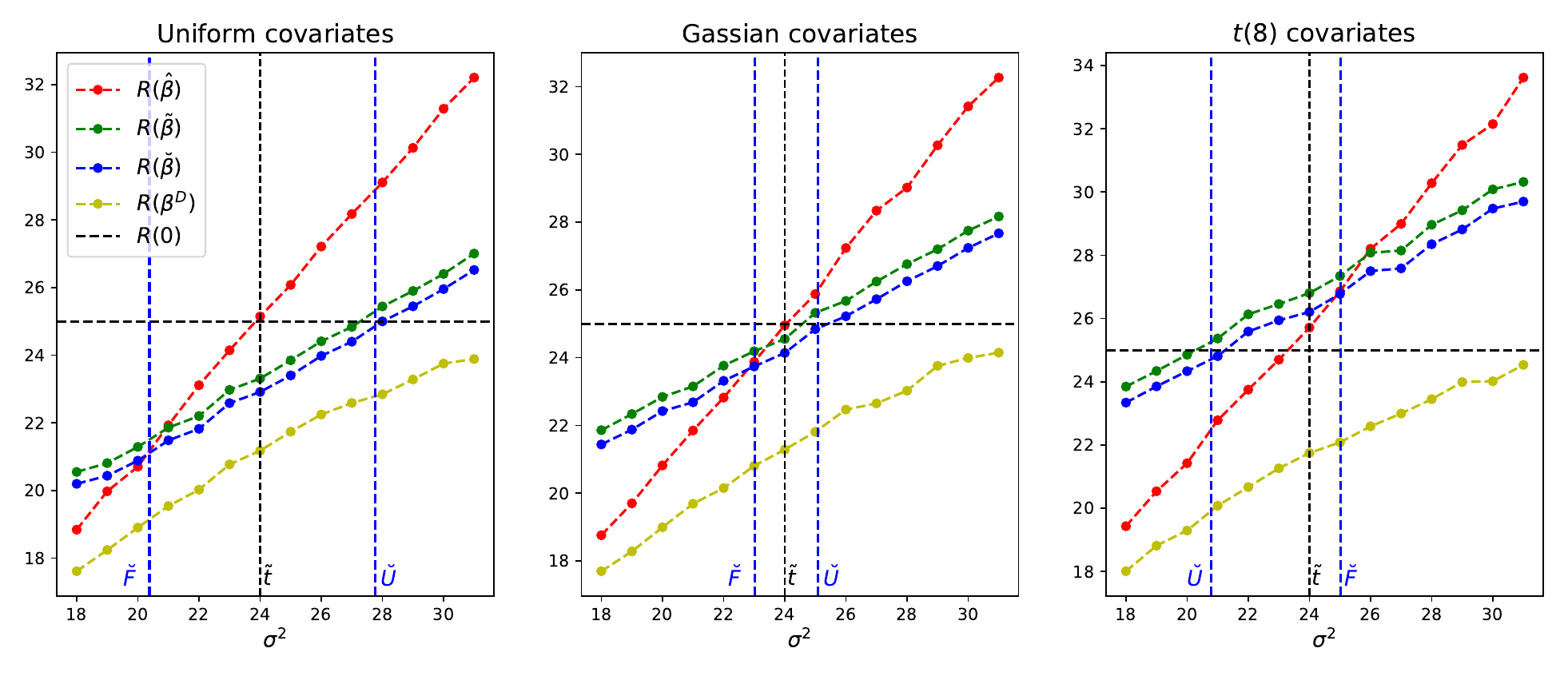}
\caption{Results of simulations for true linear model with with random-$\beta$ link, for Gaussian (middle), Uniform (left), and $t(8)$ covariates (right). }
\label{f2} 
\end{figure}

\subsection{Misspecified linear model}\label{Accounting for bias}
We now consider the case where $E[y|x]=f(x)$ for some function $f$, and the fitted model assumes that $E[y|x]=x^T\beta$. In this case, the bias term of each estimator, after subtracting the constant term $\left(  \mathbb{E}\left[ f(x_0)   \right] \right)^2$, can be written as follows:
\begin{align*}
    B_f(\hat{\beta}) =& \frac{1}{n} \text{tr}\left( H \mathbb{V}ar_X\left( (X^TX)^{-1}A_X  \right) \right) + \mathbb{V}ar_{x_0} \left( x_0^T\mathbb{E}_X[ (X^TX)^{-1}A_X]-  f(x_0)   \right),\\
    B_f(\tilde{\beta}) =&  \frac{1}{n} \text{tr}\left( H^{-1}\mathbb{V}ar_X(A_X ) \right)+ \mathbb{V}ar_{x_0} \left( x_0^TH^{-1} \mathbb{E}_X[A_X]-  f(x_0)   \right),\\
    B_f(\Breve{\beta}) =&  \frac{1}{n} \text{tr}\left( H^{-1}\mathbb{V}ar_X(C_X ) \right)   + \mathbb{V}ar_{x_0} \left( x_0^TH^{-1} \mathbb{E}_X[C_X]-  f(x_0)   \right), \\
     B_f(0) =& \mathbb{V}ar_{x_0}\left(   f(x_0)  \right),
\end{align*}
where $A_X=X^Tf(X)$, and $C_X=\widehat{\mathbb{C}ov}\left( X, f(X) \right)$.

We can use the above expressions to derive the formulas for the threshold values:
\begin{align*}
    R_f(\tilde{\beta})< R_f(\hat{\beta}) &\iff \sigma^2 > \frac{n}{ \text{tr}\left(QH  \right)-p} \left[       B_f(\tilde{\beta}) -  B_f(\hat{\beta})\right] := \tilde{F}_f(\beta),\\
    R_f(\Breve{\beta})< R_f(\hat{\beta}) &\iff \sigma^2 > \frac{n}{ \text{tr}\left(QH  \right)- \frac{n-1}{n}p } \left[       B_f(\Breve{\beta}) -  B_f(\hat{\beta})\right] := \Breve{F}_f(\beta).
\end{align*}
We can see that in this scenario, $\Breve{F}_f(\beta)$ can get negative values when $ B_f(\Breve{\beta}) < B_f(\hat{\beta})$, meaning that  $R_f(\Breve{\beta})< R_f(\hat{\beta})$ for every value of $\sigma^2$. On the other hand:
\begin{align*}
    R_f(\tilde{\beta})< R_f(0) &\iff \sigma^2 < \frac{n}{p} \left[ B_f(0)  -  B_f(\tilde{\beta})  \right] := \tilde{U}_f(\beta),\\
    R_f(\Breve{\beta})< R_f(0) &\iff \sigma^2 < \frac{n}{n-1}\frac{n}{p} \left[ B_f(0)  -  B_f(\Breve{\beta})  \right] := \Breve{U}_f(\beta),
\end{align*}
meaning that the value of $\Breve{U}_f(\beta)$,  can be negative when the model is highly biased.
\subsection{Simulations for misspecified linear model}\label{OLS_Biased_Simulations}
We perform the same experiments as in Section \ref{OLS_Und_Simulations}, but with Gaussian and Uniform covariates only, and different mechanisms for the mean function $f(x)$. The two generating models for the mean function $f(x) = \mathbb{E}[y|x]$ are as follows.

\begin{itemize}
\item Low bias. $f(x) = \beta \sum_{j=1}^p x_j +\delta |x_j| \hspace{1mm},\hspace{1mm} \beta \in \{0.25,0.5,\cdots ,1.5 \} \hspace{1mm},\hspace{1mm} \delta=0.2 $. 
\item High bias. $f(x) = \beta \sum_{j=1}^p x_j +\delta |x_j| \hspace{1mm},\hspace{1mm} \beta \in \{0.25,0.5,\cdots ,1.5 \} \hspace{1mm},\hspace{1mm} \delta=0.4 $. 
\end{itemize}

For every one of the four data-generating mechanisms, we present the unsupervised  estimation of the threshold-values as calculated for any value of $\beta$ according to the formulas in Section \ref{Accounting for bias}. We mark the estimated threshold-values $\Breve{U}(\beta=1.5)$ and $\Breve{F}(\beta=1.5)$  (only when positive) by vertical dashed lines on the supervised simulations plots. For simplicity, we denote by $R$ the reducible error $B+V$, throughout this subsection.         

In the results (Figure \ref{f3}), we can see good agreement between the unsupervised simulations and the prediction error curves in the supervised simulations. In practice, we can identify the usefulness of the semi-supervised estimator $\Breve{\beta}$ for any combination $(\beta,\delta,\sigma^2)$ of interest.   
We can see that $R(\Breve{\beta})$ is substantially lower than $R(\tilde{\beta})$ in all four scenarios. Moreover, we can see that the Gaussian covariates setting does not comply with the Sandwich phenomenon of Theorem 1 in this case. In fact, $\tilde{U}<\tilde{F}$ making $\tilde{\beta}$ the worst estimator in the interval $[\tilde{U},\tilde{F}]$. On the other hand, $\tilde{U}>\tilde{F}$ for Uniform covariates in both cases of low and high bias. We can also see that the threshold value $\Breve{F}$ is negative in the high bias scenario, and the estimator $\Breve{\beta}$ is better than $\hat{\beta}$ for any value of $\sigma^2$. Importantly, $R(\beta^D)$ is uniformly lower than $R(\hat{\beta})$ in all four scenarios with substantial improvement in the high bias scenario, even though the decision rule assumes an unbiased model and uses no prior assumptions over the $(\beta,\delta,\sigma^2)$ combination. 

\begin{figure}[h!]  
\centering
\includegraphics[width=16cm]{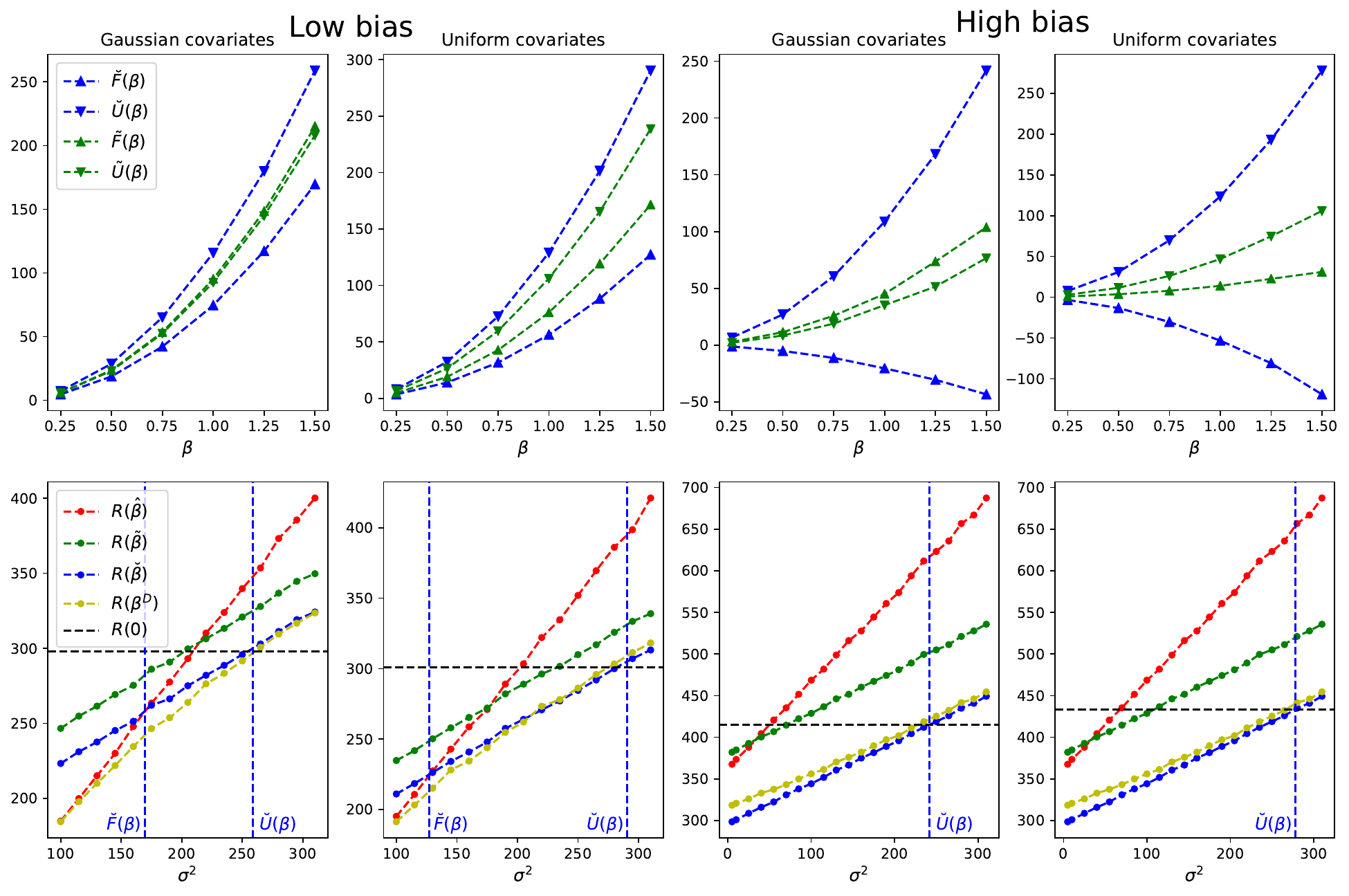}
\caption{Results of simulations for misspecified linear model. The two left plots present the low bias scenario, and the two right high bias. The unsupervised estimates are presented at the top, and supervised simulations at the bottom.}
\label{f3} 
\end{figure}

\section{Semi-supervised GLM-ERM} \label{Semi-supervised GLS-ERM}

\subsection{Semi-supervised gradient descent}

In this section we analyze GLM-ERM for general monotone increasing link-function $g$. For this case, we already defined the estimators $\hat{\beta}$, $\tilde{\beta}$ and $\Breve{\beta}$ in Equations (\ref{e1}), (\ref{e2}) and (\ref{e3}) respectively. In order to analyze the solutions of these optimization problems, we shall now define the gradients for each one of the objective functions. For the supervised procedure, the gradient $\hat{S}(\beta)$, for some vector $\beta \in \mathbb{R}^p$, can be written as $\hat{S}(\beta) =\frac{1}{n}X^T(\mu-Y)$, where $\mu \in \mathbb{R}^n$, with $\mu_{i}=g(x_i^T\beta)$. Under the innocuous assumption of exchangeability  between expected value and derivative, we can write the \textit{semi-supervised gradients}:
\begin{equation*}
    \tilde{S}(\beta) =  \mathbb{E}_x\left[  g(x^T\beta)x \right] - \frac{1}{n}X^TY =    \frac{1}{n} \left( \mathbb{E}_X\left[X^T\mu\right] -X^TY \right) ,
\end{equation*}
\begin{equation*}
    \Breve{S}(\beta) =  \frac{1}{n}\mathbb{E}_X\left[X^T\mu\right] -\left(\mathbb{E}_X[X]\cdot \overline{Y}+ \widehat{\mathbb{C}ov}(X,Y)\right),
\end{equation*}
where $\mathbb{E}_X\left[X^T\mu\right]=  n\mathbb{E}_x\left[ g(x^T\beta)x \right]$.

The notation $\mathbb{E}_X[\cdots]$ stands for the expected value over all possible random covariate matrices $X\in  \mathbb{R}^{n \times p}$ and it is used in order to simplify further discussion. In practice, if we use some sort of gradient descent algorithm to find $\tilde{\beta}$ or $\Breve{\beta}$, the learning procedure utilizes the unlabeled data in every iteration to calculate the semi-supervised gradient at the current point. For example, we will (accurately) estimate $ \mathbb{E}_X\left[X^T\mu\right]$ at the current point $\beta$ by the statistic $(n/m) \sum_{i=1}^m  g(z_i^T\beta)z_i$.

We can see that the objective functions $\hat{L}$, $\tilde{L}$, and $\Breve{L}$ are all convex w.r.t. $\beta$ by looking at the Hessian of each one of them:
\begin{align*}
    \hat{L}''(\beta) &= \frac{1}{n}X^TDX \hspace{2mm};\hspace{2mm} \tilde{L}''(\beta)= \Breve{L}''(\beta) =  \frac{1}{n}\mathbb{E}_X\left[X^TDX \right],
\end{align*}
where $D$ is $n\times n$ diagonal matrix with the terms $D_{ii}=g'(x_i\beta)$. The matrices above are S.P.D since $g$ is a monotone increasing function. Moreover, the unique solution for each one of the loss functions satisfies: $\hat{S}(\hat{\beta}) =\tilde{S}(\tilde{\beta}) = \Breve{S}(\Breve{\beta}) = \textbf{0}$.

We will use the above properties combined with more assumptions in order to analyze and compare the predictive performance of the three estimators. However, any sort of gradient descent algorithm can be applied in a semi-supervised fashion, without any assumptions over the true model or the distribution of the data, as long as the unlabeled data is taken into account in the calculation of the gradients.

\subsection{Predictive performance: approximate analysis}\label{Theoretical analysis}
In the context of GLM theory, the log-likelihood is usually approximated by a quadratic function for optimization and inference. Some works like \cite{lawless1978efficient}, \cite{minkin1983assessing}, and \cite{kredler1986behaviour}, discussed the bounds of the approximation error in terms of the model parametrization. In the latter, \cite{kredler1986behaviour} showed that in general, the nonquadratic tail is of the same order as the quadratic one, for an arbitrary sample size $n$. Nevertheless, quadratic approximation is considered common practice both for theoretical analysis and practical uses like optimization and derivation of confidence intervals.

The supervised optimization objective we presented here, $\hat{L}(\beta)$, is equivalent to the negative log-likelihood in a canonical GLM model, and therefore it is natural to approximate the loss function $L$ by a quadratic function, in order to extend the result from the linear model. Yet, in this work we view the function $g$ more as an activation function associated with modern machine learning models, rather than probabilistic function driven from classical statistical theory. Therefore, we adopt the methodology of quadratic approximation from GLM theory, but analyze the generalization error under the setting of constant conditional variance rather than model-based noise. In Section \ref{D1} we show that the results can be generalized to other models for the conditional variance.   

Recall that $L(\beta;x,y)=G(x^T \beta)+x^T \beta y$, the magnitude of the approximation error depends on the magnitude of $g''=G'''$ which is zero in the linear model. If the link function satisfies that $g''$ is bounded by small value, then $L$ should be reasonably well approximated by a quadratic function over a fair sized region around some point $\beta$. Taking for example the ReLU and Leaky ReLU functions, which are popular in deep learning models \citep{sharma2017activation}, we can see that $g''=0$ at every point except $0$, where it is not defined. It means that if $sign(x^T \beta_1)=sign(x^T \beta_2)$, then the quadratic approximation of $L(\beta_2;x,y)$ around $\beta_1$ is exact. On the other hand, if $sign(x^T \beta_1)\neq sign(x^T \beta_2)$, then a polynomial approximation of any order will fail to converge. 

Another modern activation function is the ELU function, introduced by \cite{clevert2015fast}, and can be written as follows: $g(z;a)=min\left\{a \left(e^{z/a}-1 \right),max\left(0,z\right)  \right\}$. We can see that $g''(z)=0$ when $z>0$, and $g''(z)=e^{z/a}/a \in (0,1/a)$ when $z<0$, meaning that the error of the quadratic approximation depends on the parameter $a$. We will use the ELU function to demonstrate the feasibility of deriving the same threshold values as for the OLS model, at an accuracy that changes with the value of the parameter $a$.

In order to achieve theoretical insights and compare between $\hat{\beta}$, $\tilde{\beta}$, and $\Breve{\beta}$, we assume first that the true model satisfies: $\mathbb{E}[y|x]=g(x^T\beta)$ for some $\beta \in \mathbb{R}^p$, and that the quadratic approximation of the loss function $L$ is arbitrarily good for the link function under discussion. Secondly, we extend the distributional assumptions from Section \ref{Semi-supervised OLS}, taking the assumptions that $(X^TDX)^{-1}$ exists with probability $1$, and the relevant expectations are well defined. The quadratic approximation of $\hat{L}(\hat{\beta})$ around the real $\beta$, can be written as follows:  
\begin{equation*}
    \hat{L}(\hat{\beta}) \approx L(\beta)+(\hat{\beta}-\beta)^TL'(\beta) +\frac{1}{2}(\hat{\beta}-\beta)^TL''(\beta)(\hat{\beta}-\beta).
\end{equation*}
Differentiating  both sides according to $\hat{\beta}$, since $\hat{L}'(\hat{\beta})=\hat{S}(\hat{\beta}) =\textbf{0}$, we get:
\begin{align*}
    \textbf{0} &\approx \hat{L}'(\beta)- \hat{L}''(\beta)\beta+ \hat{L}''(\beta)\hat{\beta} \implies\\
    \hat{\beta} &\approx  \beta - \left(\hat{L}''(\beta)\right)^{-1}\left[\hat{L}'(\beta) \right] 
    = \beta - \left(X^TDX\right)^{-1}X^T(\mu-Y):= \beta -\hat{a}.
\end{align*}
In the same manner we can show that:
\begin{align*}
    \tilde{\beta}  &\approx   \beta -H^{-1}\left(\mathbb{E}_X[X^T\mu] -X^TY\right) :=  \beta -\tilde{a} \\
    \Breve{\beta} &\approx \beta - H^{-1}\left(\mathbb{E}_X[X^T\mu] -n\left(\mathbb{E}_X[X]\cdot \overline{Y}+ \widehat{\mathbb{C}ov}(X,Y) \right) \right):=  \beta -\Breve{a},
\end{align*}
where $H= \mathbb{E}_X[X^TDX]$. We can see that $\mathbb{E}[\hat{a} | X] = 0$, which means that $\hat{\beta}$ is (approximately) unbiased estimator of $\beta$ for any covariate matrix $X$. On the other hand, for $\tilde{a}$ and $\Breve{a}$, only the unconditional expected value is equal to zero:
\begin{align*}
    \mathbb{E}[\tilde{a}] 
    = \mathbb{E}_X\left[H^{-1}\left(\mathbb{E}_X[X^T\mu] -X^T\mathbb{E}[Y|X]\right)   \right] =H^{-1}\left( \mathbb{E}_X[X^T\mu] - \mathbb{E}_X[X^T\mu]   \right) =\textbf{0}.\nonumber
\end{align*}
We note that $\mathbb{E}[\tilde{a}|X]$ can be nonzero as $\mathbb{E}_X[X^T\mu] -X^T\mu$ may be nonzero. In the same manner we can show that $\mathbb{E}[\Breve{a}]= 0$. We will use these results to explicitly write the prediction error of each one of the estimators.

Assume we have some estimator $\dot{\beta}$ of $\beta$, of the form: $\dot{\beta}=\beta- \dot{a}$, where $\dot{a}$ is a random vector generated by the training set $T=(X,Y)$, and $\mathbb{E}_T[\dot{a}]= 0$. The quadratic approximation for the mean out-of-sample loss of $\dot{\beta}$ can be written as follows:
\begin{align*}
    R(\dot{\beta}) 
    \approx R(\beta)+ \frac{1}{2} \mathbb{E}_{T,x_0}\left[(\dot{\beta}-\beta)^TL''(\beta,x_0)(\dot{\beta}-\beta)\right] = R(\beta) +\frac{1}{2n}\mathbb{E}_{T}\left[ \dot{a}^TH\dot{a} \right], 
\end{align*}
where $R(\beta)= \mathbb{E}_{x_0,y_0}\left[ L(\beta,x_0,y_0) \right ]= \mathbb{E}_{x_0}\left[G(x_0^T\beta) - x_0^T\beta g(x_0^T\beta)\right ]$.

Focusing on the term $\mathbb{E}_{T}\left[ \dot{a}^TH\dot{a} \right]$, we use the fact that $\mathbb{E}[Y|X]=\mu$, and  $\mathbb{E}[YY^T|X]=\mu \mu^T +I_n\sigma^2$ (constant conditional variance), to show that:
\begin{align*}
    \mathbb{E}_T\left[ \tilde{a}^T H  \tilde{a} \right] 
    &= \text{tr}\left(H^{-1}\mathbb{V}ar_X(X^T\mu) \right)+\sigma^2 \text{tr}\left( H^{-1}\mathbb{E}[X^TX] \right).
\end{align*}
The left term in the expression above does not depend on $\sigma^2$ and can viewed as the bias term $B(\tilde{\beta})$, and the right term can viewed as the variance term $V(\tilde{\beta})$, multiplied by $2n$. In the same manner we can show that:
\begin{align*}
    \mathbb{E}_T\left[ \Breve{a}^T H  \Breve{a} \right] &= \text{tr}\left(H^{-1}\mathbb{V}ar_X\left( n\mathbb{E}[X]\overline{\mu} + n\widehat{ \mathbb{C}ov}(X,\mu) \right) \right) \nonumber\\
    &+ \frac{n-1}{n}\sigma^2 \text{tr}\left( H^{-1}\mathbb{E}[X^TX] \right) := 2n\left[B(\Breve{\beta})+V(\Breve{\beta})\right] .
\end{align*}
On the other hand, for the supervised estimator we find that:
\begin{align*}
     \mathbb{E}_T\left[ \hat{a}^T H  \hat{a} \right] 
     &= \sigma^2 \text{tr}\left( \mathbb{E}_X\left[(X^TDX)^{-1}X^TX(X^TDX)^{-1}  \right]H \right) :=\sigma^2 \text{tr}\left(QH  \right):= 2n V(\hat{\beta}), \nonumber
\end{align*}

Assuming that the approximation error is negligible in the context of comparing between $R(\Breve{\beta})$, $R(\tilde{\beta})$, and $R(\hat{\beta})$, then the comparison between the three estimators depends only on the magnitude of $\mathbb{E}_{T}\left[ \dot{a}^TH\dot{a} \right]$. We conclude that as long as $\text{tr}\left(QH  \right)$ is greater than $\text{tr}\left( H^{-1}\mathbb{E}[X^TX] \right)$, we have a bias-variance trade-off between the supervised learning and the SSL methods. We can use the unlabeled data to ensure that this condition holds, and if so, we can write the lower threshold values as follows:
\begin{align*}
     R(\tilde{\beta})< R(\hat{\beta}) 
     & \iff \sigma^2 > \frac{\text{tr}\left(H^{-1} \mathbb{V}ar_X(X^T\mu) \right)}{ \text{tr}\left(QH  \right)-\text{tr}\left( H^{-1}\mathbb{E}[X^TX] \right)} := \tilde{F}(\beta),\\
     R(\Breve{\beta})< R(\hat{\beta}) 
     & \iff \sigma^2 > \frac{\text{tr}\left(H^{-1}\mathbb{V}ar_X\left( n\mathbb{E}[X]\overline{\mu} + n\widehat{ \mathbb{C}ov}(X,\mu) \right) \right)}{ \text{tr}\left(QH  \right)-\frac{n-1}{n}\text{tr}\left( H^{-1}\mathbb{E}[X^TX] \right)} := \Breve{F}(\beta).
\end{align*}
This result generalizes the previous result for the OLS model. Note that taking $D=I_n$ and $\mu=X\beta$, we get the same expressions as in the OLS model. 

The out-of-sample loss of the null model can be written as follows:
\begin{align*}
    R(0) &= \mathbb{E}\left[ L(\mu_0,x_0,y_0) \right] = \mathbb{E}_{x_0} \left[ G(g^{-1}(\mu_0)) -g^{-1}(\mu_0) \mathbb{E}[y_0|x_0] \right] = G(g^{-1}(\mu_0)) -g^{-1}(\mu_0)\mu_0.
\end{align*}
We can now use $R(0)$ to write the upper threshold values as follows:
\begin{align*}
     R(\tilde{\beta})< R(0) 
     &\iff  \sigma^2 < \frac{ 2n\left[ R(0)- R(\beta) \right]  - \text{tr}\left(H^{-1} \mathbb{V}ar_X(X^T\mu) \right)}{ \text{tr}\left( H^{-1}\mathbb{E}[X^TX] \right)} := \tilde{U}(\beta),\\
    R(\Breve{\beta})< R(0) 
     & \iff \sigma^2 < \frac{ 2n\left[ R(0)- R(\beta) - B(\Breve{\beta}) \right] }{ \frac{n-1}{n}\text{tr}\left( H^{-1}\mathbb{E}[X^TX] \right)} := \Breve{U}(\beta).
\end{align*}
The above formulas for the threshold values can be used to identify combinations $(\sigma^2,\beta)$ where the SSL is useful for improving performance of GLM-ERM models. In the next subsection we demonstrate the above insight in empirical study, concluding that the methodology of unsupervised thresholds estimation is indeed accurate in the context of GLM.

In order to define the adaptive estimator $\beta^D$, we need a suitable estimator of $\sigma^2$ for this setting. We provide an approximated unbiased  estimator, based on the expectation of the quadratic approximation of $RSS(\hat{\beta})$. We find that:
\begin{align*}
    \mathbb{E}_{X Y}\left[RSS(\hat{\beta}) \right] & \approx \sigma^2\left[ n-2p+\text{tr}\left( \mathbb{E}_X\left[X^TX(X^TDX)^{-1}X^TD^2X   (X^TDX)^{-1}\right]\right) \right],
\end{align*}
and we suggest to evaluate the matrix $D$ at $\Breve{\beta}$, resulting in the following estimator: 
\begin{align*}
    \hat{\sigma}^2=RSS(\hat{\beta})/\left(n-2p+\text{tr}\left( \mathbb{E}_X\left[X^TX(X^TDX)^{-1}X^TD^2X   (X^TDX)^{-1}\right]\right) \right) \hspace{1mm};\hspace{1mm} D_{ii}=g'(x_i^T\Breve{\beta}).
\end{align*}
We note that if $D=I_n$, this estimator coincides with standard estimator of the linear model. Simulations show that this estimator is indeed an almost-unbiased estimator of $\sigma^2$ in the experiments setting.  



\subsection{Simulations for GLM-ERM}\label{GLM_Simulations}
We perform the same experiments as in Section \ref{OLS_Und_Simulations}, but with Gaussian covariates only, and different mechanisms for the mean function $f(x)$. The two generating models for the mean function $f(x) = \mathbb{E}[y|x]$ are ELU$(\beta \sum_{j=1}^p x_j;1)$, and ELU$(\beta \sum_{j=1}^p x_j;4)$, Moreover, we set $p=10$, and denote by $R$ the reducible error, as in Secs. \ref{OLS_Und_Simulations}, \ref{OLS_Biased_Simulations} ,subtracting the fixed component $R(\beta)$.



For every random training set $(X,Z,Y)$, we implement the classical Newton-Raphson method for fitting $\hat{\beta}$, and a semi-supervised version of it for fitting $\tilde{\beta}$ and $\Breve{\beta}$. The method for fitting $\Breve{\beta}$ is summarized by the following updating step:  
\[\Breve{\beta}^{(t+1)} =  \Breve{\beta}^{(t)} -\left(\frac{1}{m}H_{m,t}\right)^{-1} \left(\frac{1}{m} Z^Tg(Z\Breve{\beta}^{(t)})  -\overline{ Z}\cdot \overline{Y} -  \widehat{\mathbb{C}ov}(X,Y)  \right),  \]
where $H_{m,t} = Z^T D_{m,t} Z$, and $D_{m,t}$ is $m\times m$ diagonal matrix with terms $[D_{m,t}]_{ii}= g'(z_i^T\Breve{\beta}^{(t)})$.




In the results (Figure \ref{f4}), we can see a reasonable agreement between the unsupervised estimations ($\Breve{F}$ and $\Breve{U}$), and the supervised results, but not as good as in the OLS models. We attribute this to the error of the quadratic approximation. Compared to the OLS model with Gaussian covariates, here $R(\Breve{\beta})$ is substantially lower than the non-SSL competitors over a wide range of $\sigma^2$. We can also see that $R(\Breve{\beta})$ is uniformly lower than $R(\tilde{\beta})$, and more importantly, that $R(\beta^D)$ is uniformly lower than $R(\hat{\beta})$.

\begin{figure}[h!]  
\centering
\includegraphics[width=16cm]{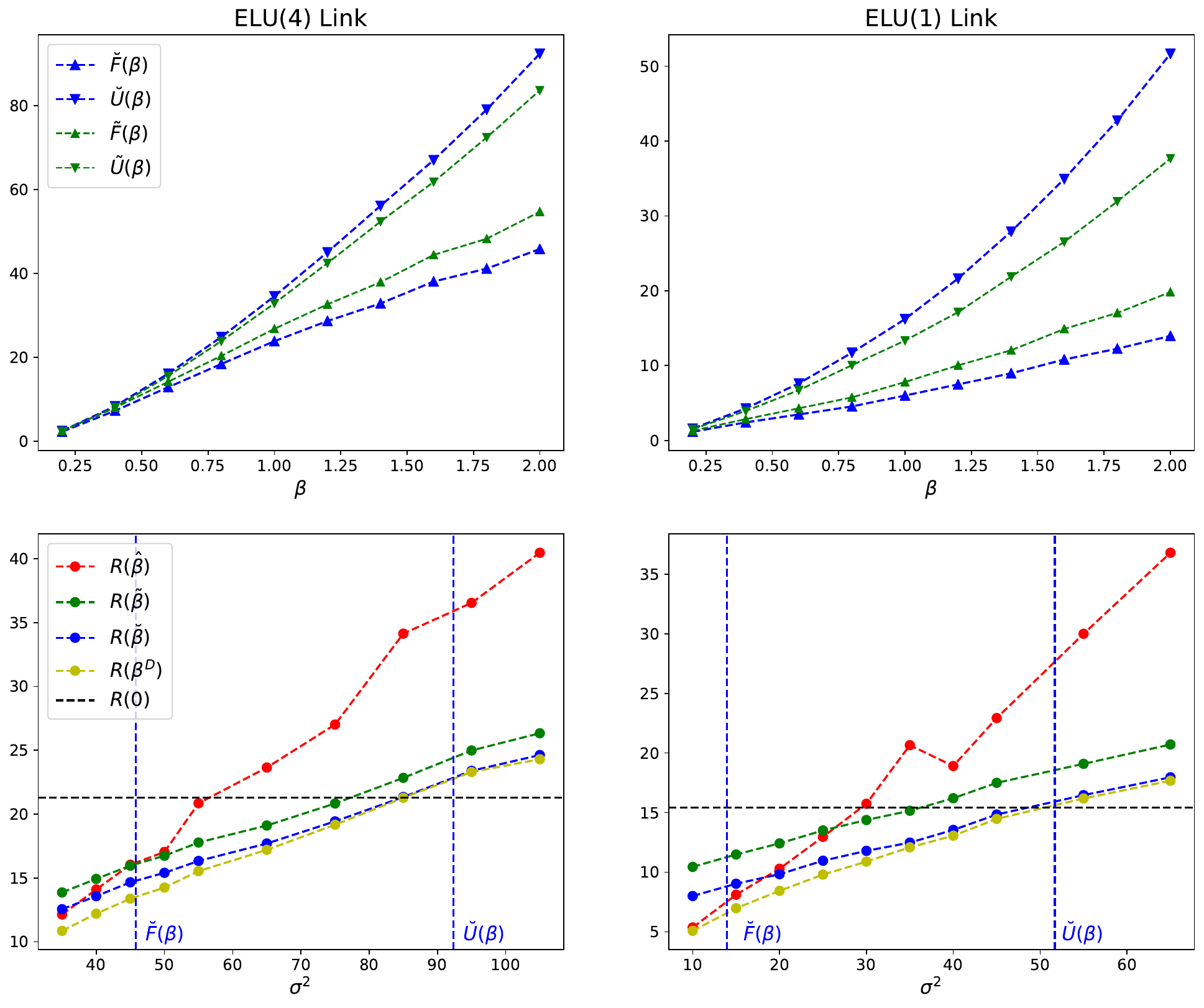}
\caption{Results of simulations for ELU link model with Gaussian covariates.}
\label{f4} 
\end{figure}

\section{Discussion} \label{Discussion}

\subsection{Summary}

In this work, we presented a general methodology for using unlabeled data to design semi-supervised variants of the ERM learning process. We focused on generalized linear models, and showed that it is possible to identify from the data the usefulness of the SSL in reducing the prediction error. We provided evidence that in some cases, having unlabeled data can lead to substantial improvement in prediction by applying the suggested methodology. In the classical linear regression model with Gaussian covariates or asymptotic setting, we provided a theoretical result, stating that the suggested estimators $\tilde{\beta}$, $\Breve{\beta}$ fail to achieve substantial improvement over the supervised model, except where the null model is superior to both supervised and SSL approaches. However the adaptive empirical estimator we propose, $\beta^D$, might achieve substantial improvement in this setting with noisy data while ensuring no deterioration when the noise is low. In all our experiments $\Breve{\beta}$ was better than the intuitive $\tilde{\beta}$, this is rigorously proven in Theorem \ref{Normal_Covariates_Thm} for the above setting. More generally this property holds if $ \mathbb{C}ov \left( x_{ij}g(x_{l}^T\beta), x_{ik}g(x_{l}^T\beta)   \right)  \leq \mathbb{C}ov \left( x_{ij}g(x_{i}^T\beta), x_{ik}g(x_{i}^T\beta)   \right),$ for any $i\neq l$, and $j,k \in \{1\cdots p\}$, which is intuitively true and supported by the simulations.

The generalized linear models discussed in this work can be viewed as a single-layer neural network, and the optimization algorithm presented here can easily be extended to networks with more than one layer. Although theoretical analysis of the usefulness of semi-supervised deep models is beyond the scope of this work, we find it a very interesting topic for future study. Other possible extensions to the scenarios studied in this work, that can be relevant for improving real-life predictive models, are discussed below in brief and can be relevant for future study.

\subsection{Non-constant conditional variance} \label{D1}
Throughout this work we assumed a constant conditional variance, $\mathbb{V}ar(y|x)=\sigma^2$, and the derivation of out-of-sample loss was according to this assumption. The given results can be generalized to other models for the conditional variance as long as they are taken into account in the derivation of $\mathbb{E}_{T}\left[ \dot{a}^TH\dot{a} \right]$. For example, a common assumption in GLM theory is that $\mathbb{V}ar(y|x)=g'(x^T\beta)\sigma^2$, where $\sigma^2$ is the dispersion parameter, naturally equal to $1$. In this case, we use $\mathbb{E}[YY^T|X]=\mu \mu^T +D\sigma^2$, to show that:
\begin{align*}
    \mathbb{E}_T\left[ \tilde{a}^T H  \tilde{a} \right] 
    &= \text{tr}\left(H^{-1}\mathbb{V}ar_X(X^T\mu) \right)+\sigma^2p,\\
    \mathbb{E}_T\left[ \hat{a}^T H  \hat{a} \right] &= \text{tr}\left( \mathbb{E}\left[(X^TDX)^{-1} H (X^TDX)^{-1}X^TDX\right]\sigma^2 \right)=\sigma^2 \text{tr}\left(QH  \right), 
\end{align*}   
where $Q=E\left[(X^TDX)^{-1}\right]$. Since $\text{tr}\left(QH  \right) \geq p$, we have a bias-variance trade-off between the supervised and the semi-supervised estimators and the expressions of the threshold values can be found. This example demonstrates the flexibility of the suggested methodology regrading the assumptions on the true model. Different assumptions will result in different expressions for the threshold values. Then, the estimation of these values using the set of unlabeled data is straightforward.  

\subsection{Different criterion for model comparison}
Throughout this work, we assume that the out-of-sample loss function is the same as the training loss function (hence ERM). In general, a model can be fitted using the loss function $L$ for the purpose of convexity, and be judged according to another criterion. For example, assume that we fit a GLM model to find an estimator $\dot{\beta}$, but we are interested in the out-of-sample squared error. In this case, regardless of the link function $g$, we can write the out-of-sample loss as follows:
\begin{align*}
    R(\dot{\beta}) 
     =& \sigma^2 +   \mathbb{E}_{X,x_0}\left( \mathbb{E}\left[g(x_0^T\dot{\beta})|X,x_0  \right] -g(x_0^T\beta) \right)^2 +  \mathbb{E}_{X,x_0}\left[ \mathbb{V}ar\left(g(x_0^T\dot{\beta})|X,x_0  \right) \right] \nonumber \\
    =&\sigma^2 + B(\dot{\beta})+V(\dot{\beta}). 
\end{align*}
We can derive the relevant approximated expressions for $R(\hat{\beta})$ and $R(\tilde{\beta})$ according to the assumed conditional variance. In the case of constant conditional variance, we can write:
\begin{align*}
    B(\hat{\beta}) &\approx  0 \hspace{2mm};\hspace{2mm} V(\hat{\beta}) \approx  \frac{\sigma^2}{n}\text{tr}\left(H_2Q \right),\\
   B(\tilde{\beta}) &\approx  \frac{1}{n}\text{tr}\left(H^{-1}H_2H^{-1}\mathbb{V}ar_X(X^T\mu)\right) \hspace{2mm};\hspace{2mm} V(\tilde{\beta}) \approx  \frac{\sigma^2}{n}\text{tr}\left(H^{-1}H_2H^{-1}\mathbb{E}[X^TX] \right),
\end{align*}
where $Q=\mathbb{E}\left[(X^TDX)^{-1}X^TX(X^TDX)^{-1}\right]$, and $H_2=\mathbb{E}\left[X^TD^2X\right]$.

The relevant threshold values can be calculated according to the above expressions, and the effectiveness of the unlabeled data can be identified  for this particular case. This is another flexibility of the suggested methodology that can be further investigated.

\subsection{Adding regularization}
Adding regularization terms (like ridge or lasso) can be done in a straightforward way in the semi-supervised optimization argument. In turn, the semi-supervised gradient can be modified and calculated according to the regularization method and parameters, resulting in different fitted estimators. Initial experiments showed that the prediction error can be improved by adding ridge regularization to semi-supervised GLM-ERM model. However, a dedicated theoretical analysis is required in order to derive the threshold values and identify the usefulness of the unlabeled data in regularized modeling.  

\bibliography{main}


\end{document}